\documentclass{aims}
\usepackage{amsmath}
  \usepackage{paralist}
  \usepackage{graphics} 
  \usepackage{epsfig} 
\usepackage{graphicx}  \usepackage{epstopdf}
 \usepackage[colorlinks=true]{hyperref}
\hypersetup{urlcolor=blue, citecolor=red}

\def\currentvolume{X}
 \def\currentissue{X}
  \def\currentyear{200X}
   \def\currentmonth{XX}
    \def\ppages{X--XX}

\newtheorem{theorem}{Theorem}[section]

\newtheorem{lemma}[theorem]{Lemma}
\newtheorem{proposition}{Proposition}

\theoremstyle{definition}

\newtheorem{remark}{Remark}

\usepackage{bm}
\usepackage{dsfont}
\usepackage{float}

\newtheorem{assumption}{Assumption}

\newcommand{\W}{{\bm W}}

\newcommand{\w}{{\bm w}}
\newcommand{\x}{{\bm x}}

\newcommand{\D}{\mathcal{D}}
\newcommand{\A}[0]{~~\text{and}~~}
\newcommand{\argmax}[1]{\mathop{\text{argmax}}_{#1}}
\newcommand{\norm}[1]{\left|#1\right|}

\newcommand{\inner}[2]{\left\langle#1,#2\right\rangle}

\newcommand{\set}[1]{\left\{#1\right\}}
\newcommand{\E}[2]{\mathop{\mathbb{E}}_{#1}\left[#2\right]}
\let\P\undefined
\newcommand{\P}[2]{\mathop{\mathbb{P}}_{#1}\left[#2\right]}
\usepackage{tikz}
\usepackage{pgfplots}
\pgfplotsset{compat=1.15}
\usetikzlibrary{calc,intersections}

\usepackage{booktabs,color}
\newcommand*\circled[1]{\tikz[baseline=(char.base)]{
            \node[shape=circle,draw,inner sep=2pt] (char) {#1};}}

\title[Running heading with forty characters or less] 
      {Global Convergence and Geometric Characterization of Slow to Fast Weight Evolution in Neural Network Training for Classifying Linearly Non-Separable Data}

\author[Ziang Long, Penghang Yin and Jack Xin]{}


 \email{zlong6@uci.edu}
 \email{pyin@albany.edu}
 \email{jxin@math.uci.edu}

\thanks{$^*$ Corresponding author}

\begin{document}
\maketitle

\centerline{\scshape Ziang Long$^*$}
\medskip
{\footnotesize
 \centerline{Department of Mathematics, University of California, Irvine, Irvine, CA 92697, USA}
} 

\medskip
\centerline{\scshape Penghang Yin}
\medskip
{\footnotesize
 \centerline{Department of Mathematics and Statistics, State University of New York at Albany, Albany, NY 12222, USA}
}

\medskip
\centerline{\scshape Jack Xin}
\medskip
{\footnotesize
 \centerline{Department of Mathematics, University of California, Irvine, Irvine, CA 92697, USA}
}

\bigskip


\begin{abstract}
In this paper, we study the dynamics of gradient descent in learning neural networks for classification problems. Unlike in existing works, we consider the linearly non-separable case where the training data of different classes lie in orthogonal subspaces. We show that when the network has sufficient (but not exceedingly large) number of neurons, (1) the corresponding minimization problem has a desirable landscape where all critical points are global minima with perfect classification; (2) gradient descent is guaranteed to converge to the global minima. Moreover, we discovered a geometric condition on the network weights so that when it is satisfied, the weight evolution transitions from a slow phase of weight direction spreading to a fast phase of weight convergence. The geometric condition says that the convex hull of the weights projected on the unit sphere contains the origin. 
\end{abstract}

\section{Introduction}
Deep neural networks (DNN) have achieved remarkable performances in image and speech classification tasks among other AI applications in recent years; for examples, see \cite{hochreiter1997long,imagnet_12,faster_rcnn,silver2016mastering}. Although there have been numerous theoretical contributions to understand their success, the learning process in the actual network training remains largely  empirical. One interesting phenomenon
is that over-parametrized DNN's trained 
by stochastic gradient descent generalize \cite{bound1,land1} instead of overfitting the training data contrary to conventional statistical learning. Though several convergence results are proved in the over-parameterized regime for deep networks \cite{du2,ll1,zhu1}, the network weights move only in a small neighborhood of the random initialization and so their dynamics are very localized. Partly, this may be attributed to the exceedingly large number of neurons in convergence theory, far surpassing what is used in practice where the weights evolve significantly from random start through hundreds of epochs in training to reach best prediction accuracy. 

Our work here addresses how the weights evolve towards a global minimum of loss function as the number of neurons increases from the feature dimension (the least necessary) to the over-parametrized regime. To facilitate analysis, our model network structure is motivated by  
\cite{bru1} on classifying linearly separable data. We instead study a linearly non-separable multi-category classification problem with an emphasis on the dynamics of weights in terms of the two time scales of evolution and a geometric characterization of the transition time. Our training data of the two classes will lie in 
orthogonal sub-spaces, which extends the data configuration in \cite{bru2} where 
the subspace of each class is one dimensional for an XOR detection problem. Orthogonality of input data from the two classes implies that  the training process in each class can be analyzed independently of the other. In the one-dimensional case \cite{bru2}, each weight update does not increase the loss on any sample point. In the multi-dimensional case here, we find that during gradient descent weight update, it is not possible that the loss is non-increasing in the point-wise sense (on each input data). Instead, the population loss is decreasing (i.e. in the sense of expectation). The population loss here is based on the hinge loss function and the network activation function is ReLU. Under a mild non-degenerate data condition, we prove that all critical points of our non-convex and non-smooth  population loss function are global  minima. Similar landscapes (a local minimum is a global minimum) are known for deep networks with activation functions that are either strictly convex \cite{liang},  or real analytic and strictly increasing \cite{land1}. 

\subsection{Prior Works and Our Contributions}
In DNN training, one observes that the network learning consists of alternating phases: plateaus where the validation error remains fairly constant and periods of rapid improvement where a lot of progress is made over a few epochs. Prior to our work, \cite{conv1} studied slow and fast weight dynamics in a solvable model while minimizing a binary cross entropy or hinge loss function on linearly separable data. 
In the regression context, \cite{brutzkus2017globally} came across such two time-scale phenomenon in training a two-linear-layer convolutional network with prescribed ground truth and unit Gaussian input data. This particular data assumption makes it possible to readily derive the closed-form expressions of the population loss and gradient, and then analyze the energy landscape and convergence of the gradient descent algorithm. 


 

In this work, we study network weight dynamics in training a one-hidden-layer ReLU network via hinge loss minimization on multi-category classification of linearly non-separable data lying in $n$ orthogonal sub-spaces. Our main contributions are:  
\begin{itemize}
    \item We discovered a geometric condition (GC) to characterize the transition time $T$ from the first (slow) phase of weight evolution to 
    the second fast weight convergence. 
    The condition says that the convex hull of the 
    weights on the unit sphere contains the origin, see Fig. \ref{simplex} for an illustration. Equivalent geometric conditions are also derived (Lemma 1). In the first (slow) phase, the weight directions spread out over the unit sphere to satisfy GC.  
    
    \item We obtain upper bound on $T$
    in terms of data distribution function provided that the network weights are uniformly bounded during training which we observed numerically.  
    
    \item We give probabilistic bounds on the validity of geometric condition for random initialization, which suggests that the larger the number of neurons, the more likely GC holds and the earlier the fast phase of evolution begins. 
    
    \item We prove the global convergence of gradient descent training algorithm under the uniformly bounded weight assumption. In case of positive network bias, we prove a global Lipschitz gradient property of the loss function and sub-sequential convergence of weights to a global minimum. In case of zero network bias, we prove that the loss function has Lipschitz gradient away from the origin and is piece-wise $C^1$. 
    
    \item We prove that all critical points of the population loss function are global minima under a non-degenerate data condition.    
    
    \item We provide numerical examples to substantiate our theory, extend the data assumption, and illustrate the weight dynamics as the network size increases towards the over-parametrized regime. We visualize the feature and weight vectors in DNNs  on MNIST data in connection with our model findings.

\end{itemize}

\noindent{\bf Organization.} In section 2, we introduce the settings of the classification problem, including the assumptions on the data and network architecture. In section 3, we state the  main results regarding the convergence guarantee of the gradient descent algorithm for training the neural net in the cases of with and without a bias term in the linear layer. In section 4, we present preliminaries about the landscape of the training loss function. The convergence analysis of main results will be sketched in section 5. In section 6, we substantiate our theoretical findings with numerical simulations. All the technical proofs are detailed in the appendix.
\medskip

\noindent{\bf Notations.} We denote by $\mathcal{S}^{d-1}$ the unit sphere in $\mathbb{R}^d$, and $\left|\mathcal{S}^{d-1}\right|$ the area of the unit sphere in the corresponding dimension. For any finite dimensional linear space $V\subseteq \mathbb{R}^d$, we define $V^k$ to be the collection of matrices of form $\left[\bm x_1,\cdots,\bm x_k\right]\in\mathbb{R}^{d\times k}$, where $\bm x_j\in V$ is the $j$-th column vector. For any set $\mathcal{X}$, $\mathds{1}_{\mathcal{X}} (x) = 1$ if $x\in\mathcal{X}$ else $0$, is the indicator function of $\mathcal{X}$. For any vector $\bm x\in \mathbb{R}^d$, we denote $|\bm x|$ be the $\ell_2$ norm of $\bm x$. For a matrix $\W=[\bm w_1,\cdots,\bm w_k]\in\mathbb{R}^{d\times k}$,
$|\W|:=\sum_{j=1}^k\left|\bm w_j\right|$ is the column-wise $\ell_2$-norm sum.


\section{Problem Setup}
In this section, we consider the multi-category classification problem in the $d$-dimensional space $\mathcal{X}=\mathbb{R}^{d}$.
 Let $\mathcal{Y}=[n]:=\{1,2,\cdots,n\}$ be the set of labels, and $\set{\D_i}_{i=1}^n$ be $n$ probabilistic distributions over $\mathcal{X}\times\mathcal{Y}$. Throughout the theoretical analysis of this paper, we make the following assumptions on the data:
\begin{enumerate}
	\item \textbf{(Separability)}  There are $n$ orthogonal subspaces $V_i\subseteq\mathcal{X}$ for $i\in[n]$ with $\mathop{\bigoplus}_{i=1}^nV_i=\mathcal{X}$,  such that  $$\mathop{\mathbb{P}}_{(\bm x,y)\sim\mathcal{D}_i}\left[\bm x\in V_i\A y=i\right]=1.$$
    
    \item \textbf{(Boundedness of data)} For $i\in[n]$, There exist positive constants $m_i$ and $M_i$, such that $$\mathop{\mathbb{P}}_{(\bm x,y)\sim\mathcal{D}_i}\left[m_i\leq\left|\bm x\right|\leq M_i\right]=1.$$
    
    \item \textbf{(Boundedness of p.d.f.)} For $i \in [n]$, let $p_i$ be the probability density function of distribution $\mathcal{D}_i$ restricted on $V_i$. For any $\bm x\in V_i$ with $m_i<|\bm x|<M_i$, it holds that $$0<p_{\text{min}}\leq p_i(\x)\leq p_{\text{max}}<\infty.$$ 
\end{enumerate}

Later on, we denote $\mathcal{D}$ to be the evenly mixed distribution of $\mathcal{D}_i$'s. For notation simplicity, we let $m=\min\limits_{i\in[n]}m_i$, $M=\max\limits_{i\in[n]}M_i$ and $d_i=\text{dim}V_i$.

We consider a two-layer neural network with $k$ hidden neurons. Denote 
by $\W=\left[\bm w_1,\cdots,\bm w_k\right]\in \mathbb{R}^{d\times k}$ the weight matrix in the hidden layer. For any input data $\bm x\in\mathcal{X}=\mathbb{R}^{d}$, we have
$$h_j=\inner{\w_j}{\x}-b_j\A f_i=\sum_{j=1}^n v_{i,j}\sigma(h_j)$$
and the neural net outputs 
\begin{equation}\label{net1}
f\left(\W; \bm x\right)=\bm V\sigma\left(\bm W^\top\x\right)=\left[f_1,\cdots,f_n\right]^\top,
\end{equation}
where $\sigma := \max(\cdot, 0)$ is the ReLU function acting element-wise, and the bias $b_j\geq 0$ and $\bm V=\left(v_{i,j}\right)$ are constants. 
Throughout this paper, we assume the following:
\begin{assumption}\label{v}
$\bm V=\left(v_{i,j}\right)\in\mathbb{R}^{k\times n}$ satisfies 
\begin{enumerate}
    \item For any $i\in[n]$, there exists some $j\in[n]$ such that $v_{i,j}>0$.
    \item If $v_{i,j}>0$ then $v_{r,j}<0$ for all $r\not=i$ and $r\in[n]$.
    \item There exists some constant $v>0$ such that $|v_{i,j}|=v$.
\end{enumerate} 
\end{assumption}
One can show that as long as $k\geq n$, such $\bm V$ can be constructed easily. 

The prediction is given by the maximum coordinate index of the network output 
$$\hat{y}\left(\W;\x\right)=\argmax{i\in[n]}f_i,$$
ideally $\hat{y}(x)=i$ if $x\in V_i$. The classification accuracy in percentage is the frequency that this occurs (when network output label $\hat{y}$ matches the true label) on a validation data set. 
Given the data sample $\{\x, y\}$, the associated hinge loss function reads 
\begin{equation}\label{sample_loss}
l(\W; \{ \x, y \}) :=  \sum_{i\not=y}\max\left\{0,1-f_y+f_i\right\}.
\end{equation}
For network training, we consider the gradient descent algorithm with step size $\eta>0$
\begin{equation}\label{gd}
\W^{t}= \W^{t-1}-\eta\nabla l(\W^{t-1})
\end{equation} 
to solve following population loss minimization problem
\begin{equation}\label{loss}
\min_{\W\in\mathbb{R}^{d\times k}}\; l\left(\W\right) = \mathop{\mathbb{E}}_{\{ \x, y \}\sim\mathcal{D}}\left[ l\left(\W; \{\x, y\}\right)\right],
\end{equation}
where the sample loss function $l\left(\W; \{\x, y\}\right)$ is given by (\ref{sample_loss}).
Let $l_i$ be the population loss function of data type $i$. More precisely, 
$$l_i(\W)=\mathop{\mathbb{E}}_{(\bm x,y)\sim\mathcal{D}_i}\left[l\left(\W;\set{\x,y}\right)\right]=\E{(\x,y)\sim\mathcal{D}_i}{\sum_{r\not=i}\sigma\left(1-f_i+f_r\right)}.$$
Thus, we can rewrite the loss function as 
$$l(\W)=\frac{1}{n}\sum_{i=1}^nl_i(\W).$$
Note that the population loss function
$$
l_i(\W)=\sum_{r\not=i}\int_{\set{f_i<f_r+1}}\left(1-f_i+f_r\right)p_i\left(\bm x\right)\;d\,\bm x
$$
has no closed-form solution even if $p$ is a constant function on its support. We cannot use closed-form formula to analyze the learning process, which makes our work different from many other works. 
\section{Main Results}
Although (\ref{loss}) is a non-convex optimization problem, we show that under mild conditions, the gradient descent algorithm (\ref{gd}) converges to a global minimum with zero classification error. 
Specifically, we consider two different networks with a positive bias $b_j>0$ (Theorem \ref{main1}) and without a bias (Theorem \ref{main2}), respectively. For both cases, we have the fact that any critical point of problem (\ref{loss}) is a global minimum (Proposition \ref{globalmin}). The key difference between these two cases is that the population loss function has Lipschitz continuous gradient (Lemma \ref{lipgrad}) when $b_j>0$, whereas this desirable property does not hold otherwise. For the latter case $b_j=0$, we present a totally different analysis based on a geometric condition proved to emerge during the training process (Proposition \ref{phase_1}). Under this geometric condition, the objective value converges zero (Proposition \ref{phase_2}).



\begin{theorem}\label{main1}
Assume $0<\sum_{j=1}^k b_j<1$ and assumption \ref{v} holds in (\ref{net1}), and $\{\W^{t}\}$ generated by the algorithm (\ref{gd}) are bounded uniformly in $t$. If there exists $(\bm x,i)\sim\mathcal{D}_i$ and some indices $j\in[k]$, such that $v_{i,j}>0$ and $\left|\inner{\w_j^0}{\x}\right|>b_j$, then there exists some $\eta_0\left(v,k,b,p_{\text{max}},n,M\right)>0$ such that for the learning rate $\eta<\eta_0$, $\lim\limits_{t\rightarrow\infty}l_i\left(\W^t\right)=0$ and 
$$\lim_{t\rightarrow\infty}\mathop{\mathbb{P}}_{\{\x,i\}\sim\mathcal D_i}\left[\hat{y}\left(\W^{t};\x\right)\neq i\right]=0, \;\; i\in[n], \; \forall \, n \geq 2.$$
\end{theorem}

\begin{proof}[Proof of Theorem \ref{main1}]
By Lemma \ref{decompose}, we only need to prove the convergence of the simplified network (\ref{net2}).
From Lemma \ref{lipgrad}, we know $l_i(\W)$ has Lipschitz gradient. We can assume for any $\W_1,\W_2\in V_i^{k}$, we have
$$\left|\nabla l_i(\W_1)- \nabla l_i(\W_2)\right|\leq L\left|\W_1-\W_2\right|.$$
As long as we take $\eta<\frac{L}{2}$ in algorithm (\ref{gd}), we know
\begin{equation}\label{deloss}
l_i(\W^{t+1})
\leq l_i(\W^{t})-\left(\eta-\frac{\eta^2L}{2}\right)\left|\nabla l_i(\W^{t})\right|^2
\leq l_i(\W^{t}).
\end{equation}
Hence, $l_i(\W^{t})$ is monotonically decreasing. Therefore, for any convergent subsequence $\{\W^{t_k}\}$ with the limit $\W_0$,
 there exists $l_0\geq 0$ such that
$$\lim_{t\rightarrow\infty}l\left(\W^{t}\right)=\lim_{k\rightarrow\infty}l\left(\W^{t_k}\right)=l_0.$$
Now, we can take subsequence and limit on both side of equation (\ref{deloss}), we get
$$
l_0\leq l_0 - \left(\eta-\frac{\eta^2L}{2}\right)\left|\nabla l_i(\W_0)\right|^2.
$$
Now, we see that $\nabla l_i(\W_0)=\bm0$. 
\end{proof}

\begin{theorem}\label{main2}
Assume $b_j=0$ and assumption \ref{v} holds in (\ref{net1}), and $\set{\W^{t}}$ generated by algorithm (\ref{gd}) is bounded by $R$ uniformly in $t$. If there exist some $(\bm x,i)\sim\mathcal{D}_i$ and some indices $j\in[k]$, such that $v_{i,j}>0$ and $\inner{\w_j^0}{\x}\not=0$, then $\lim\limits_{t\rightarrow\infty}l_i\left(\W^t\right)=0$
and
$$\lim_{t\rightarrow\infty}\mathop{\mathbb{P}}_{\{\x,i\}\sim\mathcal D_i}\left[\hat{y}\left(\W^t;\x\right)\neq i\right]=0, 
\;\; i\in[n], \; \forall \, n \geq 2.
$$
\end{theorem}


\begin{proof}[Proof of Theorem \ref{main2}]
By Proposition \ref{phase_1}, we know the iterates $\{\W^{t}\}$ stay in the first phase is bounded by $\norm{T_1}$. Combining Proposition \ref{phase_2} which shows the summation of squared loss values in phase two is bounded, the summation of all loss values in the learning process is bounded
$$\sum_{t\in T_2}l_i\left(\W^{t}\right)^2<\infty.$$
The desired result follows.
\end{proof}

\begin{remark}
The assumptions on the initialization $|\langle\bm w_j^{0},\bm x\rangle|>b_j$ in both theorems are natural. This assumption guarantees that the neuron $\bm w_j$ is activated by some input data.  Without this assumption, the algorithm suffers zero gradient and fails to update. 
\end{remark}

\begin{remark}
Theorem \ref{main2} does not explicitly require the learning rate $\eta$ to be small. However, a larger learning rate will implicitly result in a larger bound in Propositions  \ref{phase_1} and \ref{phase_2} which we used to prove Theorem \ref{main2}.
\end{remark}


\section{Preliminaries}
\subsection{Decomposition}
\begin{lemma}\label{decompose}
\begin{enumerate}
    \item For any $i\in[n]$, if $\bm W_i^*\in\mathbb{R}^{d\times k}$ solves the optimization problem
    $$\min_{\W\in V_i^{k}} \; l_i(\W),$$
    then $\bm W^*=\sum_{i=1}^n\W_i^*$ solves the original problem
    $$\min_{\W\in\mathbb{R}^{d\times k}} \; l(\W).$$
    \item If $\W'=\W-\eta\nabla_{\W}l_i(\W)$, then for any $r\not=i$, we have
    $$l(\W';\set{\x,r})=l(\W;\set{\x,r})$$
    for almost all $(\x,r)\sim\mathcal{D}_r$.
\end{enumerate}
\end{lemma}
\begin{proof}[Proof of Lemma \ref{decompose}]
Note that $\mathcal{X}=\mathbb{R}^{d}=\mathop{\bigoplus}_{i=1}^n V_i$, we decompose $\w_j=\sum_{i=1}^n\w_{j,i}$ where $\w_{j,i}\in V_i$.
	
Since $V_i$'s are orthogonal spaces, we have
$$\inner{\bm w_j}{\x}=\inner{\w_{j,i}}{\x}$$
if $\x\in V_i$. Now, assume $\x\in V_i$, let 
$$\W_i=\left[\w_{1,i},\cdots,\w_{k,i}\right],$$
we have $f(\W;\x)=f(\W_i,\x)$ and hence we get our first claim. 

On the other hand, we have
$$\nabla_{\bm w_j}l(\W;\set{\x,y})=-\sum_{i\not=y}\left(v_{y,j}-v_{i,j}\right)\mathds{1}_{\Omega_{y,i}}(\x)\mathds{1}_{\Omega_{\w_j}}(\x)\x,$$
where 
$$\Omega_{y,i}=\set{\x:f_y<f_i+1}\A\Omega_{\bm w_j}=\set{\x:\inner{\bm w_j}{\x}>b_j}.$$
Now, we see that $\nabla_{\w_j}l(\W;\set{\x,y})\in V_i$ for almost all $(\x,i)\sim \mathcal{D}_i$ so that $\W_r'=\W_r$ for all $r\not=i$ and hence our desired result follows.
\end{proof}

The optimization problem can be decomposed to $n$ independent problems of the same form i.e. the optimization of $l_i$'s. Therefore, it suffices to consider only one subproblem. Let $\W_i=\left[\bm w_1,\cdots,\bm w_k\right]\in V_i^{k}$, where $\bm w_j\in V_i=\mathbb{R}^{d_i}$, the network output for the input data $\bm x\in V_i=\mathbb{R}^{d_i}$ is given by
\begin{equation}\label{net2}
\tilde{f}^i(\W_i;\bm x)=[f_1,\cdots,f_n]^\top=\bm V\sigma\left(\W_i^\top\x\right).
\end{equation}

Networks (\ref{net1}) and (\ref{net2}) are different, since the parameters in (\ref{net1}) are $\W\in\mathbb{R}^{d\times k}$, whereas in (\ref{net2}), we have $\W_i\in V_i^{k}\cong\mathbb{R}^{d_i\times k}$. The corresponding input data are also in different intrinsic dimensions. From now on, we just focus on the loss function associated with data of Class $i$: 
$$l_i(\W)=\mathop{\mathbb{E}}_{(\bm x,y)\sim\mathcal{D}_i}\left[\sum_{r\not=i}\max\left\{0,1-f_i+f_r\right\}\right].$$

\subsection{Landscape}
The following Proposition \ref{globalmin} shows that while the loss function is non-convex, any critical point is in fact a global minimum, except for some degenerate cases. 
\begin{proposition}\label{globalmin}
Consider the neural network in (\ref{net2}). Assume $d>1$, if $\W$ is a critical point of $l_i(\W)$ and there exists some $\bm x\in V_i$ such that $f(\W;\bm x)\not=\bm0$ then we have $l_i(\W)=0$.
\end{proposition}
\begin{proof}[Proof of Proposition \ref{globalmin}]
For any $j\in[k]$, we have
$$\nabla_{\w_j}l_i(\W)=-\sum_{r\not=i}^n(v_{i,j}-v_{r,j})\E{(\x,y)\sim\mathcal{D}_i}{\mathds{1}_{\Omega_{i,r}}(\x)\mathds{1}_{\Omega_{\w_j}}(\x)\x}=\bm0.$$
Recall the definition of $\Omega_{\w_j}$, we know that each summand is a zero vector. 
By assumption \ref{v}, we know that either $v_{i,j}=v_{r,j}$ or
$$\E{(\x,y)\sim\mathcal{D}_i}{\mathds{1}_{\Omega_{i,r}}(\x)\mathds{1}_{\Omega_{\w_j}}(\x)\x}=\bm0,$$
where the latter implies $\Omega_{i,r}\cap\Omega_{\bm w_j}=\emptyset$.
Observe that 
$$f_i-f_r=\sum_{j=1}^k\left(v_{i,j}-v_{r,j}\right)\sigma(h_j),$$
we see that if there exists some $(\x,y)\sim\mathcal{D}_i$ such that $f_i-f_r<1$ then there must exist some $\x\in\Omega_{i,r}$ but this implies $\x\not\in\Omega_{\bm w_j}$ for all $j\in[k]$ such that $v_{i,j}\not=v_{r,j}$ which gives $f_i=f_r=0$. This contradicts with our assumption, so we get $f_i-f_r\geq1$ for all $\x\sim\mathcal{D}_i$ and this implies $l_i(\W)=0$.
\end{proof}
The above result holds only when the global minimum of training loss function exists. The following proposition shows that the loss function has plenty of global minima. 
\begin{proposition}\label{plenty}
Consider the network in (\ref{net2}). If the convex hull spanned by vertices $\set{\bm w_j:v_{i,j}>0}$ contains a ball centered at the origin with radius $\max\limits_{j\in[k]}\frac{1+b_j}{m_i}$, and $\set{\bm w_j:v_{i,j}<0}$ lies in a ball with radius $\min\limits_{j\in[k]}\frac{b_j}{M_i}$, then $l_i(\W)=0$. 
\end{proposition}
The above proposition shows that if number of neurons is greater than the dimension of input data, then global minimum exists. Next, we study the smoothness of the loss function. 
The following proposition shows that as long as weights are bounded away from $0$, then the loss function has Lipschitz gradient. 
\begin{lemma}\label{lipgrad}
Consider the network in (\ref{net1}) with positive bias $0<\sum_{j=1}^kb_j<1$. The loss function $l(W)$ is Lipschitz differentiable, i.e, there exists some constant $L>0$ depending on $k, \, \bm b, p_{\text{max}}, \, ,M, \, \bm V$, such that
$$
\left|\nabla l(\W_1)-\nabla l(\W_2)\right|\leq L\left|\W_1-\W_2\right|.
$$
\end{lemma}
\begin{proof}[Proof of Lemma \ref{lipgrad}]
Note that $l(\W)=\frac{1}{n}\sum_{i=1}^nl_i(\W)$, it suffice to show that each $l_i$ has Lipschitz gradient. Note that 
$$l_i(\W)=\sum_{r\not=i}\E{(\x,y)\sim\mathcal{D}_i}{\sigma\left(1-f_i+f_r\right)},$$
it suffice to show each summand has Lipschitz gradient. Now, we write the gradient of the summands as
$$
\begin{aligned}
&\nabla_{\w_j}\E{(\x,y)\sim\mathcal{D}_i}{\sigma\left(1-f_i+f_r\right)}\\
=&\E{(\x,y)\sim\mathcal{D}_i}{\nabla_{\w_j}\sigma\left(1-f_i+f_r\right)}\\
=&\E{(\x,y)\sim\mathcal{D}_i}{\mathds{1}_{\Omega_{i,r}}(\x)\mathds{1}_{\Omega_{\w_j}}(\x)\left(v_{r,j}-v_{i,j}\right)\x}.
\end{aligned}
$$
On one hand, if $v_{r,j}=v_{i,j}$, then the formula above is obviously zero and Lipschitz gradient follows. On the other hand, we can without loss of generality assume $\left|v_{i,j}-v_{r,j}\right|=1$.
For notation simplicity, we fix $i$ and $r$ and denote
$$\varphi(\W):=\E{(\x,y)\sim\mathcal{D}_i}{\mathds{1}_{\Omega_{i,r}}(\x)\mathds{1}_{\Omega_{\w_j}}(\x)\x}\A\phi(\W;\x)=f_i(\W;\x)-f_r(\W;\x).$$
Now, our desired result becomes $\varphi(\W)$ is Lipschitz in $\W$. Denote $\Omega_1=\Omega_{i,r}(\W_1)$ and $\Omega_2=\Omega_{i,r}(\W_2)$ where $\W_i=\left[\w_1^i,\cdots,\w_k^i\right]$, we only need to show there exists some constant $L$ such that
$$\norm{\varphi(\W_1)-\varphi(\W_2)}\leq L\norm{\W_1-\W_2}.$$
Note that 
$$\begin{aligned}
&\norm{\varphi(\W_1)-\varphi(\W_2)}\\
=&\norm{\E{}{\mathds{1}_{\Omega_1}(\x)\mathds{1}_{\Omega_{\bm w_j^1}}(\x)\x}-\E{}{\mathds{1}_{\Omega_2}(\x)\mathds{1}_{\Omega_{\bm w_j^2}}(\x)\x}}\\
\leq&\underbrace{\E{}{\mathds{1}_{\Omega_1\Delta\Omega_2}(\x)\norm{\x}}}_{\circled{1}}+\underbrace{\E{}{\mathds{1}_{\Omega_{\w_j^1}\Delta\Omega_{\w_j^2}}\norm{\x}}}_{\circled{2}},
\end{aligned}$$
we can deal with $\circled{1}$ and $\circled{2}$ respectively.

W.l.o.g, we assume $\epsilon=\norm{\W_1-\W_2}\leq\frac{1}{2}$.
On one hand, if $\x\in\Omega_1\Delta\Omega_2$, then we claim
$$1-\epsilon\norm{\x}\leq\phi(\W_1;\x)\leq1+\epsilon\norm{\x}$$
because
$$\begin{aligned}
&\norm{\phi(\W_1;\x)-\phi(\W_2;\x)}=\norm{\sum_{j=1}^k(v_{i,j}-v_{r,j})\left[\sigma(h_j^1)-\sigma(h_j^2)\right]}\\
\leq&\norm{\sum_{j=1}^k\inner{\w_j^1-\w_j^2}{\x}}\leq\norm{\W_1^\top\x-\W_2^\top\x}\leq\epsilon\norm{\x}.
\end{aligned}$$
Furthermore, we claim for such $\x$'s the gradient of $\phi$ is bounded away from zero. More precisely, with $\x=r\omega$ where $r=\norm{\x}$, we have
$$
\begin{aligned}
&\frac{d}{dr}\phi(\W;\x)=\left[\sum_{j=1}^k(v_{i,j}-v_{r,j})\mathds{1}_{\Omega_{\bm w_j}}(\x)\w_j\right]\omega\\
\geq&\left(\phi(\W;\x)-\sum_{j=1}^kb_j\right)/r\geq\frac{1}{M}\left(\frac{1}{2}-\sum_{j=1}^kb_j\right)=:C>0.
\end{aligned}
$$
Now, we have
$$
\begin{aligned}
&\circled{1}\leq\E{(\x,y)\sim\mathcal{D}_i}{\mathds{1}_{1-\epsilon\norm{\x}\leq\phi(\W_1,x)
\leq1+\epsilon\norm{\x}}}(\x)\norm{\x}\\
\leq&\int_{1-\epsilon\norm{\x}\leq\phi(\W_1,x)\leq1+\epsilon\norm{\x}}\norm{\x}p_{max}\,d\,\x\leq2\left(\norm{\mathcal{S}^{d_i-1}}\frac{M^{d_i}p_{max}}{C}\right)\epsilon.
\end{aligned}
$$

As for $\circled{2}$, w.l.o.g. we can assume $\norm{\w_j^1}\geq\norm{\w_j^2}$. Note that when $\norm{\w_j^1}\leq\frac{b_j}{M}$ then $h_j\leq0$ so that $\nabla_{\w_j}l_i(\W)=\bm0$ and this $\circled{2}=0$. Hence, we only need to take care of the case when $\norm{\bm w_j^1}\geq\frac{b_j}{M}$.
Note that $\norm{\w_j^1-\w_j^2}\leq\epsilon$ we know
$$\sin\theta\leq\frac{\epsilon M}{b_j},$$
where $\theta$ denotes the acute angle between $\w_j^1$ and $\w_j^2$. We have the following estimate
$$\circled{2}\leq p_{max}\frac{\epsilon M}{b_j}\norm{\mathcal{S}^{d_i-1}}.$$
Combine with $\circled{1}$, we get our desired result. 
\end{proof}

Note that Lipschitz differentiability in Lemma \ref{lipgrad} does not hold for the case $b_j=0$, as the gradient might be volatile near the origin.
\section{Convergence Analysis for Non-Bias Case}

With the Lipschitz differentiability shown in Lemma \ref{lipgrad} in the case $b_j>0$, it is not hard to prove the convergence result in Theorem \ref{main1}. In this section, we focus on the non-bias case ($b_j=0$) where the Lipschitz differentiability fails and sketch the convergence analysis. 


\begin{lemma}\label{decent}
Consider the network (\ref{net2}), $|\bm w_j^{t}|$ is non-decreasing in $t$ if $v_{i,j}>0$. For any $r>0$, choosing learning rate  $\eta<\min\left\{\frac{r}{C_pM_i^2},\frac{r}{2vnM_i}\right\}$, then $|\w_j^{t}|$ is non-increasing if $v_{i,j}<0$ and $\left|\bm w_j^t\right|>r$, where $C_p$ is a constant satisfying
$$C_p=\max_{\bm v\in V_i,a\in\mathbb{R}}\int_{\left\{\langle\bm v,\bm x\rangle=a\right\}}p_i(\x)\;d\,\x=O\left(p_{\text{max}}\,M^{d_i}\right).
$$
\end{lemma}

\begin{proof}[Proof of Lemma \ref{decent}]
We first define 
$$C_p=\max_{\bm v\in V_1,a\in\mathbb{R}}\int_{\langle\bm v,\bm x\rangle=a}p_1(\x)\;d\,S\leq M^{d-1}p_{\text{max}}.$$

Recall the definition of $\Omega_{\w_j}$, for all $\bm x\in\Omega_{\w_j}$, we have $\left\langle\bm \tilde{\bm w}_j,\x\right\rangle>0$. For $v_{i,j}>0$,
$$\left\langle\tilde{\bm w}_j,-\nabla_{\bm w_j}l_i(\W)\right\rangle=2v\,\sum_{r\not=i}\mathop{\mathbb{E}}_{\x\sim\mathcal{D}_i}\left[\mathds{1}_{\Omega_{i,r}}(\x)\mathds{1}_{\Omega_{\w_j}}(\x)\langle\tilde{\bm w}_j,\x\rangle\right]\geq0.$$

Note we have from (\ref{gd}) that
$$\bm w_j^{t+1}=\bm w_j^{t}-\eta\nabla_{\bm w_j}l_i(\W^{t}).$$
So, 
$$
\left|\bm w_j^{t+1}\right|=\left\langle\bm w_j^{t+1},\tilde{\bm w}_j^{t+1}\right\rangle
\geq\left\langle \bm w_j^{t+1},\tilde{\bm w}_j^{t}\right\rangle
\geq \left\langle \bm w_j^{t},\tilde{\bm w}_j^{t}\right\rangle
=\left|\bm w_j^{t}\right|.
$$

Let $\Omega_{i,r}^j=\Omega_{i,r}\cap\Omega_{\bm w_j}$. For $v_{i,j}<0$, we know
$$\left|\nabla_{\w_j}l_i(\W)\right|=2v\left|\sum_{r\not=i}\mathop{\mathbb{E}}\left[\mathds{1}_{\Omega_{i,r}^{j}}(\x)\x\right]\right|\leq2v M\sum_{r\not=i}\mathbb{P}\left[\Omega_{i,r}^{j}\right]^2,$$
where we omit the distribution $\x\sim\mathcal{D}_i$.

On the other hand, by definition of $C_p$, we know
$$
\left|\left\langle\nabla_{\bm w_j}l_i(\W),\tilde{\bm w}_j\right\rangle\right|
=\left|2v\sum_{r\not=i}\mathop{\mathbb{E}}_{\x\sim\mathcal{D}_i}\left[\mathds{1}_{\Omega_{i,r}^{j}}(\x)\langle\tilde{\bm w}_j,\x\rangle\right]\right|
\geq\frac{v}{C_p}\sum_{r\not=i}\mathbb{P}\left[\Omega_{i,r}^{j}\right]^2.
$$

When $0<\eta<\frac{r}{2vnM}$, we have 
$$
\left\langle\w_j-\eta\nabla_{\w_j}l_i(\W),\tilde{\w}_j\right\rangle
=\left\langle\w_j,\tilde{\w}_j\right\rangle-\eta\left\langle\nabla_{\bm w_j}l_i(\W),\tilde{\bm w}_j\right\rangle
>r-2v\eta M\sum_{r\not=i}\mathbb{P}\left[\Omega_{i,r}^{j}\right]>0.
$$

Now, we decompose $\nabla_{\w_j}l_i(\W^{t})$ into two parts, 
$$
\nabla_{\bm w_j}l_i(\W^{t}) = \underbrace{\left\langle\tilde{\w}_j^{t},\nabla_{\bm w_j}l_i(\W^{t})\right\rangle \tilde{\w}_j^{t}}_{\bm n}
+ \underbrace{\left(\nabla_{\bm w_j}l_i(\W^{t}) -\left\langle\tilde{\w}_j^{t},\nabla_{\bm w_j}l_i(\W^{t})\right\rangle \tilde{\w}_j^{t}\right)}_{\bm \nu}.
$$
So that when $\eta<\frac{r}{M^2C_p}$, we have
$$
\begin{aligned}
&\left|{\w}^{t+1}_j\right|^2\\
=&\left|\bm w_j^{t}-\eta \nabla_{\bm w_j}l_i(\W^{t})\right|^2\\
=&\left|\bm w_j^{t}-\eta\bm n\right|^2+\left|\eta \bm \nu\right|^2\\
=&\left|\bm w_j^{t}\right|^2-\eta\left(2\left|\bm w_j^{t}
\right|\left|\bm n\right|-\eta\left(\left|\bm n\right|^2+\left|\bm \nu\right|^2\right)\right)\\
\leq&\left|\bm w_j^{t}\right|^2-2v\eta\left( \frac{\left|\bm w_j^{t}
\right|}{C_p} -\eta\,M^2 \right)\sum_{r\not=i}\P{}{\Omega_{i,r}^j}^2\\
\leq& \left|\bm w_j^{t}\right|^2,\\
\end{aligned}
$$
as long as $\left|\bm w_j\right|\geq r$. Now, we get the desired result. 
\end{proof}

The above theorem provides the dynamic information of the weights. When we input data with distribution $\mathcal D_i$, the weights $\set{\bm w_j:v_{i,j}>0}$ become more useful for classification as the norm of those $\bm w_j$'s grows larger on every iteration. On the other hand, $\set{\bm w_j:v_{i,j}<0}$ serve only as noise when $v_{i,j}<0$. 
On one hand, when $\w\in\set{\w_j:v_{i,j}<0}$ has large magnitude, the learning process guarantees the decreasing of $\norm{\w}$. On the other hand, when $\w\in\set{\w_j:v_{i,j}<0}$ has a small magnitude, it shall be trapped into a small region near the origin and contribute little to classification. 


The learning process consists of two phases. In the beginning, since the weights are randomly distributed, there may well be some data that does not activate any neurons. The weights then automatically spread out. We call this process the first (slow learning) phase, during which the learning process is rather slow. The next Lemma lists four equivalent statements of a geometric condition. When the geometrical condition holds, we say that the learning process enters the second (fast learning) phase. 
\begin{figure}[htp]
    \centering
    \begin{tikzpicture}
    [
    scale=2.5,
    >=stealth,
    point/.style = {draw, circle,  fill = black, inner sep = 1pt},
    dot/.style   = {draw, circle,  fill = black, inner sep = .2pt},
    ]
    \node (O) at (0,0)[point,color=red,label = {\textcolor{red}{$\bm 0$}}]{};
    \node (w1) at (0,0.9)[point, label={left:$\tilde{\bm w}_1$}]{};
    \node (w2) at (-0.75,-0.35)[point, label={below:$\tilde{\bm w}_2$}]{};
    \node (w3) at (0.3,-0.8)[point, label={below left:$\tilde{\bm w}_3$}]{};
    \node (w4) at (0.65,0.0)[point, label={right:$\tilde{\bm w}_4$}]{};
    \draw [line width=0.6mm](O) circle (1);
    \draw[-][line width=0.4mm,color=blue] (w1) -- (w2);
    \draw[-][line width=0.4mm,color=blue] (w1) -- (w3);
    \draw[-][line width=0.4mm,color=blue] (w1) -- (w4);
    \draw[-][line width=0.4mm,color=blue] (w2) -- (w3);
    \draw[-][line width=0.4mm,color=blue] (w3) -- (w4);
    \draw[dashed][line width=0.4mm,color=blue] (w2) -- (w4);
    \draw[dotted][line width=0.6mm] (1,0) arc (0:180:1 and 0.25);
    \draw[line width=0.6mm] (-1,0) arc (180:360:1 and 0.25);
    \end{tikzpicture}
    \caption{Geometric Condition in Lemma \ref{geo} ($d=3$)}
    \label{simplex}
\end{figure}
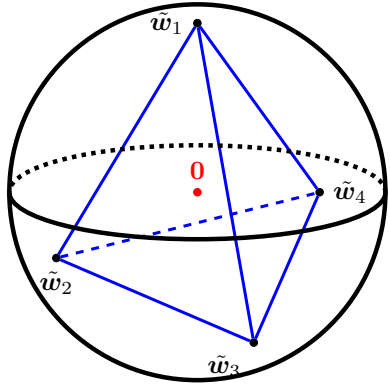
\begin{lemma}\label{geo}
Let $\tilde{\bm w}_j=\frac{\bm w_j}{\left|\bm w_j\right|}\in\mathcal{S}^{d-1}$ be $k$ points on the unit sphere, where $1\leq j\leq k$. Let  $\Lambda_W$ be the convex hull of $\left\{\tilde{\bm w}_j\right\}$ and $\left\{H_i\right\}_{i=1}^l$ be the facets of convex hull. 
The following statements are equivalent:
\begin{enumerate}
	\item For any unit vector $\bm n\in\mathcal{S}^{d-1}$, there exist some $j_1$ and $j_2$ such that $\langle\bm n,\tilde{\bm w}_{j_1}\rangle>0$ and $\langle\bm n,\tilde{\bm w}_{j_2}\rangle<0$.
	\item There exists no closed hemisphere that contains all $\tilde{\bm w}_j$.
	\item $\bm 0$ lies in the interior of $\Lambda_W$.
	\item For each $0\leq i\leq l$, we have $\bm 0$ and all $\tilde{\bm w}_j$ lie on the same side of $H_i$.
\end{enumerate}
\end{lemma}

\begin{proof}[Proof of Lemma \ref{geo}]
$(1\Leftrightarrow2)$ is trivial. 

$(1\Rightarrow3)$ Proof by contradiction. Assume $\bm 0\not\in\Lambda_W^\circ$. Since $\Lambda_W^\circ$ is a open convex set, and $\left\{\bm 0\right\}$ is a convex set, we know from geometric form of Hahn-Banach Theorem, that there exist a closed hyper-plane that separates $\left\{\bm 0\right\}$ and $\Lambda_W^\circ$. Hence, there exist a unit vector $\bm n\in\mathcal{S}^{d-1}$, such that $\langle\bm n,\tilde{\bm w}_j\rangle\geq\langle\bm n,\bm 0\rangle=0$ for all $j$, but this contradict with our assumptions in 1.

$(3\Rightarrow4)$ Assume that $H_i=\left[\langle\bm v_i,\bm w\rangle=\alpha_i\right]$, where $\alpha>0$. 
Note that the convex hull $\Lambda_W$ is a polytope with faces $H_i$. We know, $\Lambda_W^\circ$ all lies on one side of $H_i$. Since $\bm0\in\Lambda_W^\circ$, we know for all $\bm w\in \Lambda_W^\circ$ we have $\langle\bm v_i,\bm w\rangle<\alpha_i$, and hence $\langle\bm v_i,\tilde{\bm w}_j\rangle\leq\alpha_i$ for all $j$.

$(4\Rightarrow3)$ is trivial.

$(3\Rightarrow1)$ $\bm 0\in \Lambda_W^\circ$ implies there exist positive numbers $\lambda_j$, such that 
$$\sum_{j=1}^k\lambda_j\tilde{\bm w}_j=\bm 0.$$
Hence for any unit vector $\bm n$, we have
$$\sum_{j=1}^{k}\lambda_j\langle\bm n,\tilde{\bm w}_j\rangle=0.$$
Since $\tilde{\bm w}_j$ are in general position, so $\langle\bm n,\tilde{\bm w}_j\rangle$ cannot all be $0$. Hence, there must be both positive and negative terms. Now, we get the desired result. 
\end{proof}

\begin{remark}\label{prob}
If $\bm w_j$ is initialized such that $\tilde{\bm w}_j$ is uniformly distributed on $\mathcal{S}^{d-1}$, then the probability that the geometric condition (GC) in Lemma \ref{geo} holds is 
$$
{\rm P_{gc}} :={\rm Prob}(\mbox{GC holds}) = 2^{1-k}\sum_{j=d}^{k-1}\binom{k-1}{j}.
$$
In particular, at any fixed feature dimension $d$, $$\lim\limits_{k\to \infty}\, {\rm P_{gc}}= 1.$$
\end{remark}
\begin{proof}[Proof of Remark \ref{prob}]
For any $J\in\{\pm1\}^k$, we let $S_J\left(\left\{\tilde{\bm w}_j\right\}\right)=\left\{J_j\tilde{\bm w}_j:1\leq j\leq k\right\}$. 
From the choice of $\tilde{\bm w}_j$, we know all $S_J$ have the same distribution, so that
$$
{\rm Prob}_{\left\{\tilde{\bm w}_j\right\}}({\rm GC }\, {\rm holds})={\rm Prob}_{S_J}({\rm GC }\, {\rm holds})
.$$
Hence, we can simplify the probability as follows
$$
{\rm Prob}({\rm GC }\, {\rm holds})=\frac{1}{2^n}\mathbb{E}\left[\sum_J \mathds{1}_{gc}\left(S_J\left(\left\{\tilde{\bm w}_j\right\}\right)\right)\right]
.$$
We claim that $\sum_J\mathds{1}_{gc}\left(S_J\left(\left\{\tilde{\bm w}_j\right\}\right)\right)$ is a constant independent of choice of $\tilde{\bm w}_j$ as long as they are in general position. 

Let $\tilde{\bm w}_j\in\mathcal{S}^{d-1}$ where $1\leq j\leq k$. Each $\tilde{\bm w}_j$ corresponds to a subspace $H_j$ with codimension $1$ in $\mathbb{R}^d$, that is 
$$H_j=\left\{\x\in\mathbb{R}^d:\langle\tilde{\bm w}_j,\bm x\rangle=0\right\}.$$
Note that any connected region of $\left(\cup_JH_J\right)^c$ corresponds to a choice of $J$ such that the geometric condition fails. More precisely, for any given connected region $D$ of $\left(\cup_JH_J\right)^c$, we know $\langle\tilde{\bm w}_j,\bm x\rangle$ is one sign for all $\bm x\in D$, so with
$$J_j={\rm sign}\left(\langle\tilde{\bm w}_j,\bm x\rangle\right),$$
we know $\bm x$ correspond to $S_J\left(\left\{\tilde{\bm w}_j\right\}\right)$ where the geometric condition fails. Hence, the number of connected regions of $\left(\cup_JH_J\right)^c$ is same as $2^n-\sum_J\mathds{1}_{gc}\left(S_J\left(\left\{\tilde{\bm w}_j\right\}\right)\right)$.
By \cite{count}, we have
$$\sum_J\mathds{1}_{gc}\left(S_J\left(\left\{\tilde{\bm w}_j\right\}\right)\right)=2^n-2\sum_{j=0}^{d-1}\binom{k-1}{j}=2\sum_{j=d}^k\binom{k-1}{j}.$$
The desired result follows. 
\end{proof}
Per Remark 3, the more neurons the network has, higher possibility the geometric condition in Lemma \ref{geo} holds upon initialization. As a consequence, the learning process skips the first phase and goes straight to the fast learning phase. This explains why gradient descent for learning a over-parameterized network converges rapidly to a nearby critical point from random initialization.   

The following proposition gives an upper bound on the maximum number of iterations for the learning process to enter the second phase. 

\begin{proposition}\label{phase_1}
Let $b_j=0$ in (\ref{net2}), and assume that $\left|\W^{t} \right|\leq R$ for all $t$. Let $T_1$ be the set of $t$ such that $\left\{\bm w_j^t:v_{i,j}>0\right\}$ does not satisfy the geometric condition in Lemma \ref{geo}, then
$\left|T_1\right|\leq\frac{C_pR}{v\eta p_R^2}$, where $p_R$ is a positive constant. Also, the following estimate holds for $p_R$:
$$p_R=\Omega\left(\frac{p_{\text{min}}}{\sqrt{d_i}\left(M_iR\right)^{d_i}}\right)$$
where $C_p$ is the constant in Lemma \ref{decent}. More precisely, 
$$\left|T_1\right|=O\left(\frac{C_pd_iR^{2d_i+1}M_i^{2d_i}}{v\eta p_{\text{min}}^2}\right)=O\left(\frac{C_pdR^{2d+1}M^{2d}}{v\eta p_{\text{min}}^2}\right).$$
\end{proposition}
\begin{proof}[Proof of Proposition \ref{phase_1}]
W.l.o.g, assume $i=1$ and $v_{1,j}>0$ if and only if $j\in[k_1]$. Let $t\in T_1$, by Lemma \ref{geo}, we know, for any fixed $t\in T_1$, there exists a $\alpha\in[0,\frac{\pi}{2}]$ and a unit vector $\bm v\in\mathbb{R}^{d_1}$, such that $\langle\bm v,\tilde{\bm w}_j^{t}\rangle\geq\sin\alpha$ for all $j\in[k_1]$ and $\langle\bm v,\tilde{\bm w}_1^{t}\rangle=\sin\alpha$. 
W.l.o.g, we assume $\bm v=(1,0,\cdots,0)$ and $\inner{\bm v}{\tilde{\w}_1^t}=\sin\alpha$. For any non-zero $\bm x\in V_1$, we write $\tilde{\bm x}=\frac{\bm x}{\left|\bm x\right|}\in \mathcal{S}^{d_1-1}$. 

Note that $\left|\W\right|$ is bounded by $R$, we know there exist $\beta\in(0,\frac{\pi}{2}-\alpha)$ such that $\sum_{j=1}^k\left|\bm w_j^t\right|\leq R=\frac{1}{2vM_1\sin\beta}$. 
Now, w.l.o.g, we can assume $\tilde{\bm w}_1^{t}=(\cos\alpha,\sin\alpha,0,\cdots,0)$. Take $\bm n=(-\cos\alpha,\sin\alpha,0,\cdots,0)$, we know $\langle\bm n,\tilde{\bm w}_1^{t}\rangle=0$. For all $\bm x\in V_1$ such that $\langle\bm n,\tilde{\bm x}\rangle>\cos\beta$, we have
$$
\cos\beta
<\langle\bm n,\tilde{\bm x}\rangle=-\tilde{x}_1\cos\alpha+\tilde{x}_2\sin\alpha
\leq -\tilde{x}_1\cos\alpha+\sqrt{1-\tilde{x}_1^2}\sin\alpha .
$$
So, we have $$\tilde{x}_1<-\cos(\alpha+\beta).$$

Now, for all $\bm x\in V_1$ such that $\tilde{x}_1<-\cos\left(\alpha+\beta\right)$, we have 
$$
f_1(\W^{t};\bm x)-f_r(\W^{t};\bm x)\leq 2vM_1\sum_{j=1}^k\sigma\left(\left\langle\bm w_j^{t},\tilde{\bm x}\right\rangle\right)
\leq2v M_1\sum_{j=1}^k|\bm w_j^{t}|\sin\beta<1.
$$ 
So, with
$$D^1:=\left\{\bm x\in V_1:\langle\bm n,\tilde{\bm x}\rangle>\cos\beta\text{ and } \langle\tilde{\bm w}_1,\tilde{\bm x}\rangle>0\right\},$$
we have for any $r>1$
$$D^1\subset\Omega_{1,r}.$$
See Fig \ref{helper} for a intuition of $D^1$. 
\begin{figure}[htp]
    \centering
    \begin{tikzpicture}
	[
	scale=2.5,
	>=stealth,
	point/.style = {draw, circle,  fill = black, inner sep = 1pt},
	dot/.style   = {draw, circle,  fill = black, inner sep = .2pt},
	]
	\def\rad{1.25}
	\node (x1) at +(90:\rad) [label = {below right:$x_1$}] {};
	\node (x2) at +(0:\rad) [label = {below right:$x_2$}] {};
	\def\rad{1.24}
	\node (x11) at +(90:\rad) [label = {}] {};
	\node (x22) at +(0:\rad) [label = {}] {};
	\draw[->] (x11) -- (x1);
	\draw[->] (x22) -- (x2);	
	\def\rad{0.906308}
	\node (T) at +(-45:\rad) [label = {}] {};
	\def\rad{2.3662}
	\node (origin) at (0,0) [point, label = {below left:$\bm 0$}]{};
	\node (w1) at +(45:\rad) [point, label = {below right:$R\tilde{\bm w}_1$}] {};
	\draw[->][line width=0.5mm] (origin) -- (w1);
	\def\rad{1}
	\draw [line width=0.5mm](origin) circle (\rad);
	
	\node (n2) at +(45:\rad) [point, label = {right:$\tilde{\bm w}_1$}] {};
	
	\draw[->][line width=0.5mm] (origin) -- (n2);
	\node (ax) at +(0:\rad) [label = {}] {};
	\draw[dashed][line width=0.5mm] ($ (origin) ! 1.5 ! (ax) $) --  ($ (ax) ! 2.5 ! (origin) $);
	\node (ay) at +(90:\rad) [point,label = {above left:$\bm v$}] {};
	\draw[->][line width=0.5mm] (origin) -- (ay);
	\draw[dashed][line width=0.5mm] ($ (origin) ! 1.5 ! (ay) $) --  ($ (ay) ! 2.5 ! (origin) $);
	\node (A) at +(-20:\rad) [label = {}] {};
	\draw[dotted][line width=0.5mm]  ($ (origin) ! 1.5 ! (A) $) --  ($ (A) ! 2.5 ! (origin) $);
	\node (B) at +(-45:\rad) [point,label = {below:$\bm n$}] {};
	\draw[->][line width=0.5mm] (origin) -- (B);
	\node (C) at +(-160:\rad) [label = {}] {};
	\node (D) at +(-135:\rad) [label = {}] {};
	\draw[dotted][line width=0.5mm]  ($ (origin) ! 1.5 ! (B) $) --  ($ (B) ! 2.5 ! (origin) $);
	\draw[dotted][line width=0.3mm]  ($ (A) ! 1.2 ! (C) $) --  ($ (C) ! 1.2 ! (A) $);
	\draw[dotted][line width=0.3mm]  ($ (B) ! 1.4 ! (D) $) --  ($ (D) ! 1.4 ! (B) $);
	\draw[line width=0.5mm] (w1) -- (A);
	\node (E) at (intersection of origin--B and A--C)[label = {}] {};
	\draw[dotted][line width=0.5mm] ($(A) ! 2.4 ! (T)$) -- ($(T) ! 1.6 ! (A)$);
	\draw [line width=0.5mm](0,0) -- (0:.2cm) arc (0:45:.2cm);
	\draw[line width=0.5mm](25:0.3cm) node {$\alpha$};
	\draw[line width=0.5mm] (0,0) -- (-20:.2cm) arc (-20:-45:.2cm);
	\draw[line width=0.5mm](-30:0.3cm) node {$\beta$};
	\draw[line width=0.5mm] (0,0) -- (-45:.2cm) arc (-45:-90:.2cm);
	\draw[line width=0.5mm](-70:0.3cm) node {$\alpha$};
	\begin{scope}[shift={(45:2.3662)}]
		\draw[line width=0.5mm] (0,0) -- (-110:.2cm) arc (-110:-135:.2cm);
		\draw[line width=0.5mm] (55:-0.3cm) node {$\beta$};
	\end{scope}
	\node (p1) at (0.7,-0.6) [label = {}] {};
	\node (p2) at (1.1,-0.58) [label = {}] {};
	\draw[->][line width=0.5mm] (p1) -- (p2);
	\node (pt) at (1.2,-0.72) [label = {$D^1$}] {};
	\end{tikzpicture}
    \caption{2-dim section of $\mathbb{R}^d$ spanned by $\tilde{\bm w}_1$ and $\bm n$}
    \label{helper}
\end{figure}
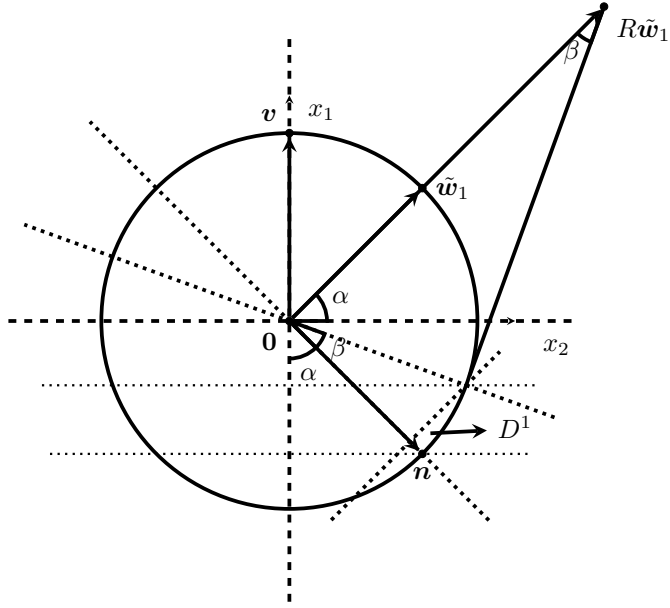
Note that $\bm n$ and $\tilde{\bm w}_1$ perpendicular to each other, and the probability distribution $p_1$ is bounded from below, so there exists a constant $p_R$, such that
$$\mathbb{P}\left[\Omega_{1,r}\cap\Omega_{\w_j^t}\right]\geq\mathbb{P}\left[D^1\right]=p_R>0,$$
and we have the following estimate for $p_R$
$$
\begin{aligned}
p_R&=\int_{D^1}\left\langle\tilde{\bm w}_1^{t},\nabla_{\bm w_1}l(\W)\right\rangle p_1\left(\bm x\right)\;d\,\bm x\\
&\geq\frac{p_{\text{min}}\left|\mathcal{S}^{d_1-2}\right|}{\left|\mathcal{S}^{d_1-1}\right|}\int_{\cos\beta}^1\left(1-y^2\right)^{\frac{d_1-3}{2}}dy\\
&=\frac{p_{\text{min}}\left|\mathcal{S}^{d_1-2}\right|}{\left|\mathcal{S}^{d_1-1}\right|}\int_0^\beta\left(\sin\theta\right)^{d_1-2}d\theta\\
&=\Omega\left(\frac{p_{\text{min}}\left(\sin\beta\right)^{d_1-1}}{(d_1-1)}\frac{\left|\mathcal{S}^{d_1-2}\right|}{\left|\mathcal{S}^{d_1-1}\right|}\right)\\
&=\Omega\left(\frac{p_{\text{min}}}{\sqrt{d_1}\left(M_1R\right)^{d_1}}\right).
\end{aligned}
.$$

Now, we know the gradient of $\bm w_1$ on $\tilde{\bm w}_1$ direction is bounded below, with same arguments in proof of Lemma \ref{decent}, we have
$$
\left|\left\langle\tilde{\bm w}_1^{t},\nabla_{\bm w_1}l_1\left(\W^{t}\right)\right\rangle\right|\geq\frac{v\,p_R^2}{C_p}.
$$
Note that 
$$\sum_{j=1}^{k_1}\left|\bm w_j^{t+1}\right|-\left|\bm w_j^{t}\right|\geq\left|\bm w_1^{t+1}\right|-\left|\bm w_1^{t}\right|\geq\eta \left|\left\langle\tilde{\bm w}_1^{t},\nabla_{\bm w_1}l\left(\W^{t}\right)\right\rangle\right|\geq\frac{v\,\eta\,p_R^2}{C_p}$$
and since $\left|\bm w_j^{t}\right|$ is non-decreasing for $j\in[k_1]$, we have
$$\sum_{t\in T_1}\sum_{j=1}^{k_1}\left|\bm w_j^{t+1}\right|-\left|\bm w_j^{t}\right|\leq R.$$
Combining the above two equations, it follows $\left|T_1\right|\leq\cfrac{C_pR}{v\,\eta\, p_R^2}$.
\end{proof}
At the beginning of the second phase, $\tilde{w}_j$ are already evenly distributed. That is, the convex hull of $\tilde{\bm w}_j$ contains the origin, and any input data must at least activate some of the neurons. As long as an input data contributes to the loss, it also contributes to the gradient. So learning process becomes faster during the second phase. The following proposition shows the loss decays  faster than before. 

\begin{proposition}\label{phase_2}
Let $b_j=0$ in (\ref{net2}). Let $T_2$ be the set of $t$ such that $\{\bm w_j^t:v_{i,j}>0\}$ satisfies the geometric condition in Lemma \ref{geo}, and that  $\left|\W^t \right|$ is upper-bounded by $R$ at all $t$. Then:
$$\sum_{t\in T_2} l_i\left(\W^{t}\right)^2\leq4\eta^{-1}vn^2C_pR^2M_i^2R.$$
\end{proposition}

\begin{proof}[Proof of Proposition \ref{phase_2}]
First, we estimate $l_i(\W^t)$ as follows
$$
\begin{aligned}
&l_i(\W^t)=\sum_{r\not=i}\E{(\x,y)\sim\mathcal{D}_i}{\sigma\left(1-f_i+f_r\right)}\\
\leq&\sum_{r\not=i}\left(2vRM_i\right)\P{(\x,y)\sim\mathcal{D}_i}{f_i<1+f_r}\\
=&\left(2vRM_i\right)\sum_{r\not=i}\P{(\x,y)\sim\mathcal{D}_i}{\Omega_{i,r}}.
\end{aligned}
$$
Now, we have
$$\sum_{r\not=i}\P{(\x,y)\sim\mathcal{D}_i}{\Omega_{i,r}}\geq\frac{l_i(\W^t)}{2vRM_i}.$$

Second, we estimate the gradient. Let $\Omega_{i,r}^j=\Omega_{i,r}\cap\Omega_{\w_j}$ and assume $v_{i,j}>0$, we have
$$\nabla_{\w_j}l_i(\W^t)=\sum_{r\not=i}2v\E{(\x,y)\sim\mathcal{D}_i}{\mathds{1}_{\Omega_{i,r}^j}(\x)\x}.$$
By the same arguments in proof of Lemma \ref{decent}, we know
$$\inner{\nabla_{\bm w_j}l_i(\W^t)}{\tilde{\w}_j^t}\geq2v\frac{\P{}{\Omega_{i,r}^j}^2}{2C_p}=\frac{v\P{}{\Omega_{i,r}^j}^2}{C_p}.$$

Next, we utilize the geometric condition which implies
$$\sum_{v_{i,j}>0}\P{}{\Omega_{i,r}^j}\geq\P{}{\Omega_{i,r}}$$
and get
$$
\begin{aligned}
&\sum_{v_{i,j}>0}\norm{\w_j^{t+1}}-\norm{\w_j^t}\geq\sum_{\substack{v_{i,j}>0\\r\not=i}}\frac{v\eta\P{}{\Omega_{i,r}^j}^2}{C_p}=\frac{v\eta}{C_p}\sum_{\substack{v_{i,j}>0\\r\not=i}}\P{}{\Omega_{i,r}^j}^2\\
&\geq\frac{v\eta}{C_p}\frac{1}{k_i}\sum_{r\not=i}\P{}{\Omega_{i,r}}^2\geq\frac{v\eta}{C_p}\frac{1}{k_i(n-1)}\left(\sum_{r\not=i}\P{}{\Omega_{i,r}}\right)^2\\
&\geq\frac{v\eta}{n^2C_p}\frac{l_i(\W^t)^2}{4v^2R^2M_i^2}=\frac{\eta}{4vn^2C_pR^2M_i^2}l_i(\W^t)^2.
\end{aligned}
$$
where $k_i$ is the number of $j$'s such that $V_{i,j}>0$.

Finaly, we combine the inequalities above and the assumption $\norm{\W^t}<R$ and get
$$\sum_{t\in T_2}l_i(\W^t)^2\leq\frac{4vn^2C_pR^2M_i^2}{\eta}\sum_{\substack{v_{i,j}>0\\t\in T_2}}\norm{\w_j^{t+1}}-\norm{\w_j^t}\leq4\eta^{-1}vn^2C_pR^2M_i^2R.$$
\end{proof}

From the above proposition, we see that in order to get $l_i\left(\W^{t}\right)<\epsilon$ for some $t\in T_2$, we only need $$|T_2|\geq\frac{4vn^2C_pR^2M_i^2R}{\eta\epsilon}.
$$
Compared with Proposition \ref{phase_1}, we see that $|T_2|$ does not rely on the dimension $d_i$, whereas the upper bound of $|T_1|$ is exponential in $d_i$.

\section{Experiments}
In this section, we report the results of our experiments on both synthetic and MNIST data. The experiments on synthetic data aim to show that {\it convergence to global minimum continues to hold if data subspaces form acute angles, going beyond the theoretical orthogonality assumption} under which convergence is observed to be the fastest.  Our theoretical results and the geometric conditions are supported by  simulations. Experiments on MNIST dataset exhibit subspace structures in data flow and  slow-to-fast training dynamics on LeNet-5. These phenomena from our model are worth further study in deep networks. 
\subsection{Synthetic Data}
Let $\left\{\bm e_j\right\}_{j\in[4]}$ be orthonormal basis of $\mathbb{R}^4$, $\theta$ be an acute angle and $\bm v_1=\bm e_1$, $\bm v_2=\sin\theta\,\bm e_2+\cos\theta\,\bm e_3$, $\bm v_3=\bm e_3$, $\bm v_4=\bm e_4$. Now, we have two linearly independent subspaces of $\mathbb{R}^4$ namely $V_1=\text{Span}\left(\left\{\bm v_1,\bm v_2\right\}\right)$ and $V_2=\text{Span}\left(\left\{\bm v_3,\bm v_4\right\}\right)$. We can easily calculate that the angle between $V_1$ and $V_2$ is $\theta$.
Next, we define
$$\hat{\mathcal{X}}_1=\left\{r\left(\cos\varphi\,\bm v_1+\sin\varphi\,\bm v_2\right):r\in S_r,\varphi\in S_\varphi\right\},$$
$$\hat{\mathcal{X}}_2=\left\{r\left(\cos\varphi\,\bm v_3+\sin\varphi\,\bm v_4\right):r\in S_r,\varphi\in S_\varphi\right\},$$
where
$$S_r=\left\{\frac{20}{j}:j\in[20]-[9]\right\}, \;
S_\varphi=\left\{\frac{j\pi}{40}:j\in[80]\right\}.$$
Let $\mathcal{X}_1$ corresponds to label $y=1$ and $\mathcal{X}_2$ corresponds to label $y=-1$. 
Since we are considering a binary classification problem,
the neural network structure can be simplified as  (\ref{net3}):
\begin{equation}\label{net3}
\tilde{f}\left(\W;\x\right)=\sum_{j=1}^{k}\sigma\left(h_j\right)-\sum_{j=k+1}^{2k}\sigma\left(h_j\right),
\end{equation}
and the prediction is given by $\hat{y}(\x)=\text{sign}\tilde{f\left(\W;\x\right)}$.
Now, the population loss becomes
$$l_i(\W):=\frac{1}{|\hat{\mathcal{X}_i}|}\sum_{\x\in\hat{\mathcal{X}_i}}\max\set{0,1+(-1)^{i}\tilde{f}\left(\W;\x\right)}.$$

In our first simulation, we set $k=4$ in (\ref{net3}) and run gradient descent (\ref{gd}) on $l_1+l_2$ with learning rate $\eta=0.1$. Fig. \ref{angle} shows the iterations it takes to converge to global minima given $\theta$ and a Gaussian noise added to $\mathcal{X}_i$'s. From this simulation, we see that the orthogonal data assumptions are only  technically needed and our convergence result holds in more general settings.

\begin{figure}[htp]
	\centering
	\includegraphics[width=0.49\linewidth]{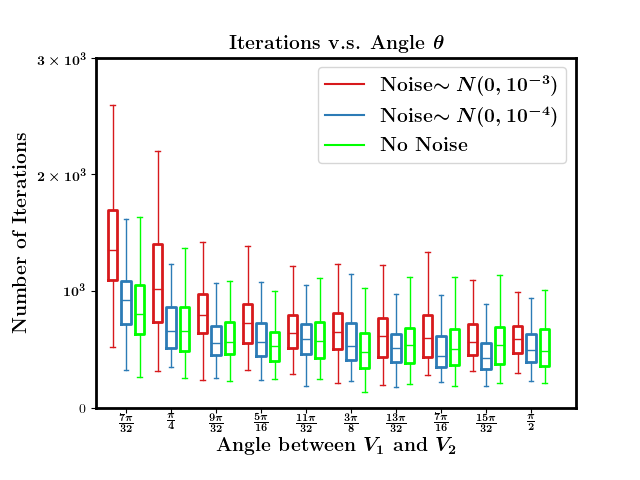}
	\caption{Number of iterations to convergence v.s. $\theta$, the anlge between subspaces $V_1$ and $V_2$.}
	\label{angle}
\end{figure}

In our second and remaining simulations of this subsection, we take $\theta=\frac{\pi}{2}$ so that $V_1$ and $V_2$ are orthogonal. Lemma \ref{decompose} suggests the learning process of $l_1$ and $l_2$ are independent, so we only simulate the training process of $l_1$ and assume $\w_j\in V_1\cong\mathbb{R}^2$.
We take entries of $\W^{0}$ to be i.i.d. standard normal i.e. $\bm w_j^{0}\sim N(\bm 0,I_2)$. We train the network (\ref{net3}) with gradient descent (\ref{gd}) in all our simulations, where learning rate $\eta=0.1$.
The left plot of Fig. \ref{iters} shows how many iterations algorithm (\ref{gd}) takes in searching for a global minima from the random initialization mentioned above. 
For each box, the red mark indicates the median, and the bottom and top of the box indicate the 25th and 75th percentiles, respectively. 
As we can see from the graph, as number of hidden neurons ($2k$) becomes larger, the algorithm (\ref{gd}) tends to need less iterations in searching for a global minima. 
\begin{figure}[htp]
	\centering
	\begin{tabular}{cc}
	\includegraphics[width=0.49\linewidth]{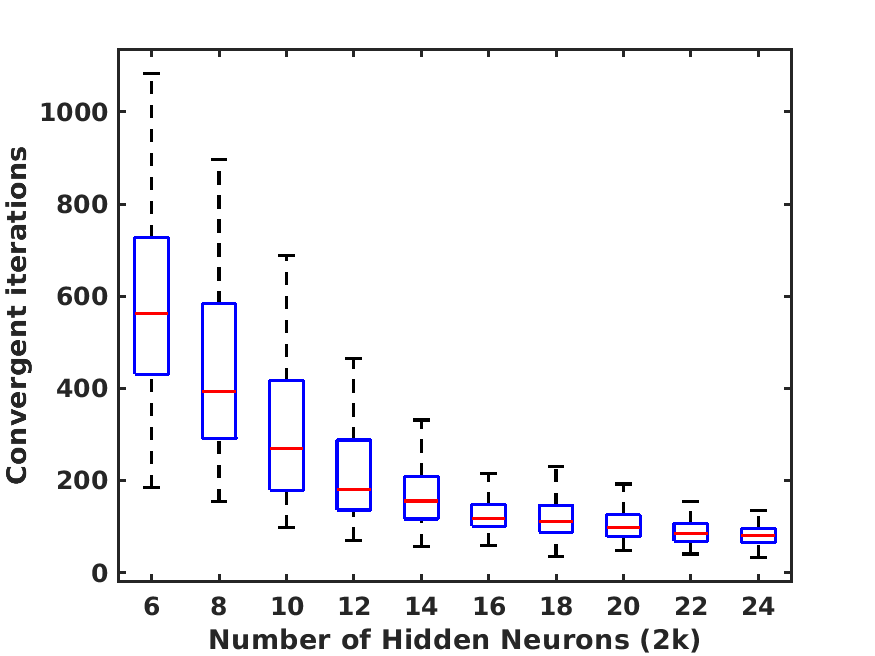}&
	\includegraphics[width=0.49\linewidth]{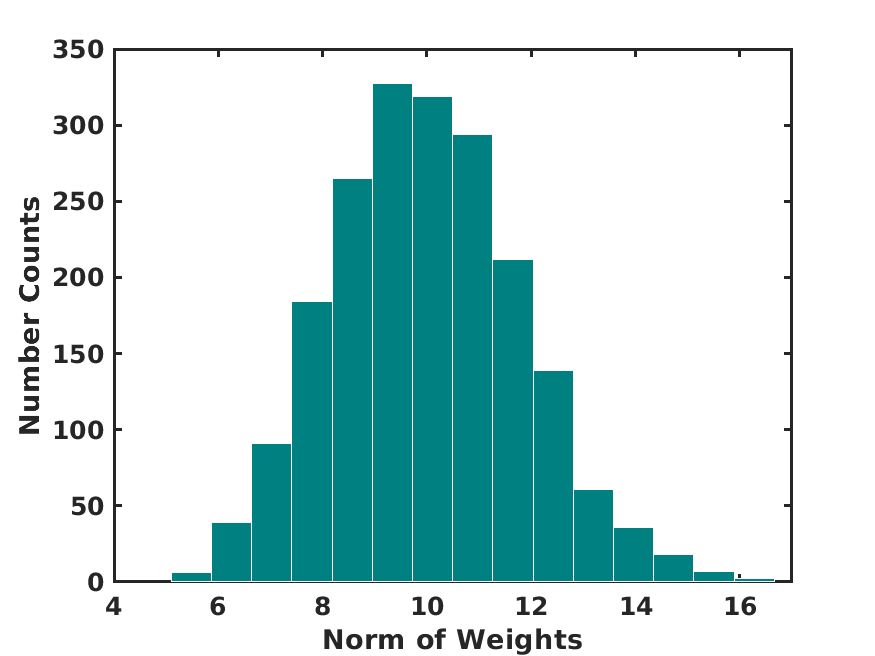}
	\end{tabular}
	\caption{Left: convergent iterations vs. number of neurons ($d=2$). Right: histogram of norm of weights: $\max\limits_{t}\left|\W^{t}\right|$ ($d=2$ and $k=4$).}
	\label{iters}
\end{figure}

In the third simulation, we compare the convergence speed with and without the geometric condition being satisfied.  
We introduce two initialization method: random initialization and half space initialization i.e. with $\hat w_{j,i}\sim N(0,1)$, random initialization takes $w_{j,i}^{0}=\hat w_{j,i}$ whereas half space initialization takes $w_{j,1}^{0}=|\hat w_{j,1}|$ and  $w_{j,2}^{0}=\hat w_{j,2}$. 
We run the algorithm for 100 times with different numbers of hidden neurons using initialization methods, and report the means and standard variances of the number of iterations in Table \ref{init}. We see from Remark \ref{prob} how the ${\rm P}_{gc}$ increases when the number of hidden neurons grows. However, the half space initialization never satisfies the geometric condition, as all the weights lie in the same half space. 
A widely believed explanation on why a neural network can fit all training labels is that the neural network is over-parameterized. Our work explained one of the reasons why over-parameterization helps convergence: it helps the weights to spread more 'evenly' and quickly after initialization. Table \ref{init} shows that when we randomly initialize, the iterations for convergence in gradient descent (\ref{gd}) come down a lot as the number of hidden neurons increases; much less so in half space initialization. 
\begin{table}[ht]
  \caption{Iterations taken ($\text{mean}\pm\text{std}$) to convergence with random and half space initializations.}
  \label{init}
  \centering
  \begin{tabular}{ccc}
    \toprule
    \# of Neurons ($2k$)  & Random Init.   & Half Space Init. \\
    \midrule
    6   & 578.90$\pm$205.43 & 672.41$\pm$226.53\\
    8   & 423.96$\pm$190.91 & 582.16$\pm$200.81\\
    10  & 313.29$\pm$178.67 & 550.19$\pm$180.59\\
    12  & 242.72$\pm$178.94 & 517.26$\pm$172.46\\
    14  & 183.53$\pm$108.60 & 500.42$\pm$215.87\\
    16  & 141.00$\pm$80.66 & 487.42$\pm$220.48\\
    18  & 126.52$\pm$62.07 & 478.25$\pm$202.71\\
    20  & 102.09$\pm$32.32 & 412.46$\pm$195.92\\
    22  & 90.65$\pm$28.01 & 454.08$\pm$203.00\\
    24  & 82.93$\pm$26.76 & 416.82$\pm$216.58\\
    \bottomrule
  \end{tabular}
\end{table}

Our fourth simulation take specifically $2k=8$. With 2000 runs we did a histogram of the maximum norm of $\W$ during the training process shown in the right plot of Fig. \ref{iters}. In fact, our third simulation suggests our boundedness assumption on $\W$ in Theorem \ref{main1} and Theorem \ref{main2} are reasonable.

\begin{figure}[htp]
  \centering
  \begin{tabular}{cccc}
    \includegraphics[width=0.245\linewidth]{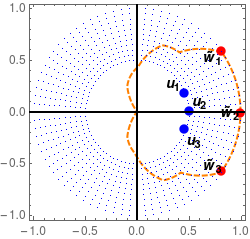}
    \includegraphics[width=0.245\linewidth]{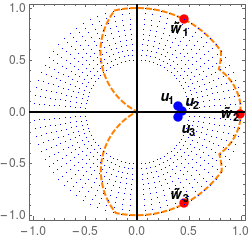}
    \includegraphics[width=0.245\linewidth]{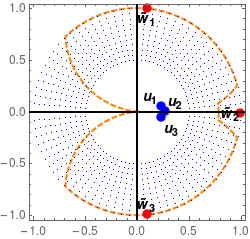}
    \includegraphics[width=0.245\linewidth]{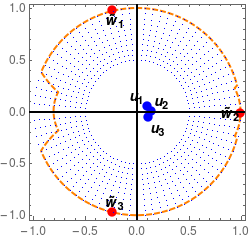}
  \end{tabular}
  \caption{Dynamics of weights: $\tilde{\bm w}_j$ and $\bm u_j$}
  \label{dynamic}
\end{figure}
In our last simulation, we take $k=3$ so that there are in total 6 hidden neurons. For notation simplicity, we denote $\bm u_j=\w_{j+3}$.
For $j\in[3]$, we plot $\tilde{\bm w}_j$'s and $\bm u_j$'s in Fig. \ref{dynamic}, where we plot $\tilde{\bm w}_j$'s instead of $\bm w_j$'s since some of $|\bm w_j|$'s are greater than one. Before algorithm (\ref{gd}) starts, the parameters in neural network (\ref{net2}) are initialized to be $$\bm w_j^{0}=\bm u_{j}^{0}=\frac{3}{4}\left(\cos\frac{\left(2-j\right)\pi}{6},\sin\frac{\left(2-j\right)\pi}{6}\right)$$ for $j\in\{1,2,3\}$. 
In Fig. \ref{dynamic}, the tiny blue points are input data under Kelvin transformation: $\x\rightarrow\x^*=\frac{\x}{|\x|^2}$. 
Take $\bm x=\frac{1}{r}\left(\cos\theta,\sin\theta\right)$ so that under Kelvin transformation $\x^*=r\left(\cos\theta,\sin\theta\right)$. For convenience, we let $\tilde{\x}=r\x=\left(\cos\theta,\sin\theta\right)$. 
The orange dashed curve has expression in polar coordinates: $$\rho(\theta) = \min\left\{1,\sigma\left(\tilde{f}\left(\W;\tilde{\x}\right)\right)\right\}.$$
Note that we are taking Hinge loss $l(\W;\{\x,1\})=0$ if and only if $\tilde{f}(\W;\x)\geq1$, i.e. $$\tilde{f}\left(\W;\tilde{\x}\right)=r\tilde{f}\left(\W;\x\right)\geq r=|\x^*|.$$
Here, in our data set $\hat{\mathcal{X}}$, all data point have norm less than one under Kelvin transformation, so $l(\W;\{\x,1\})=0$ if and only if
$\rho(\theta)\geq \left|\x^*\right|$.
This means, the blue points when surrounded by the orange dashed curve provide zero loss. In particular, when $\rho(\theta)=1$, the population loss is 0. 

\subsection{MNIST Experiments}
\begin{figure}[htp]
	\centering
	\begin{tabular}{cc}
	\includegraphics[width=0.48\linewidth]{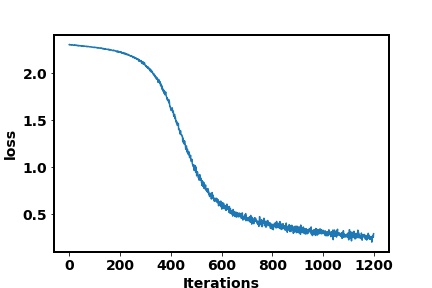}&
	\includegraphics[width=0.48\linewidth]{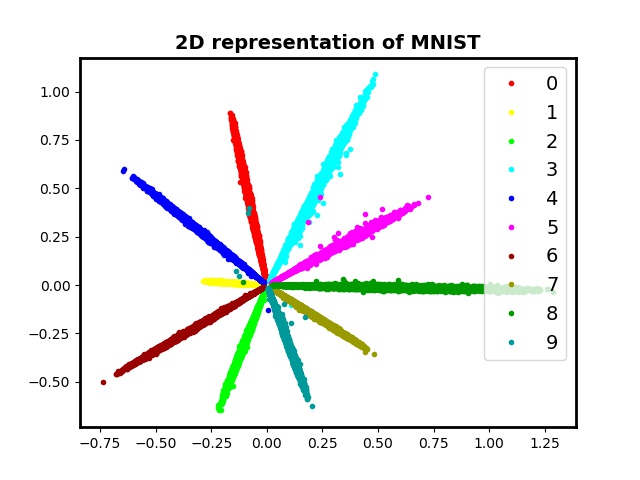}
	\end{tabular}
	\caption{{\bf Left}: Slow-to-Fast transition during  LeNet \cite{lenetversion} training on MNIST dataset. {\bf Right}: 2D projections of MNIST features from a trained convolutional neural network \cite{cosface}. The 10 classes are color coded, the feature points cluster near linearly independent subspaces. }
 	\label{MNIST}
\end{figure}
The two-phase dynamics we proved in our model does appear in deep network training on real (non-synthetic) data sets. In experiments 
on MNIST, we used a simplified version of LeNet-5 \cite{lenetversion} with 2 convolutional layers and two fully-connected (fc) layers; see \cite{Lenet-5} for the full version where F6 and output layers correspond to our fc layers. The simplified Lenet-5 is trained via stochastic gradient descent (SGD) at constant learning rate  $0.01$, batch size $1000$ and without momentum or regularization. We show in Fig. \ref{MNIST} (left) the loss value vs. iterations during training. At the early stage, the loss decays slowly, then the fast phase sets in 
after 400 iterations. 
Fig. \ref{MNIST} (left) clearly supports our theory on the slow and fast dynamics of gradient descent. 
\medskip

Network visualization helps understand its geometric properties.
In Fig. \ref{MNIST} (right), we plot 2D projections of feature vectors at input to fc layer extracted by a neural network \cite{cosface} consisting of 6  convolutional layers followed by an fc layer. Projected features from different classes cluster around linearly independent subspaces. The plot suggests that in trained deep networks, the linearly independent subspace assumption  approximately holds for the input to the fully connected layer before classification output. Similar subspace structure of high level  feature vectors 
on CIFAR-10 and enlargement of subspace angles to improve classification accuracy have been studied in \cite{LFT_2018}.     
\medskip

In Fig. \ref{gc_mnist_3d_fc1}, we show four  projections onto unit sphere $\mathcal{S}^2$ (inside randomly selected 3D subspaces) of the weight vectors of the first and second fc layers of LeNet \cite{lenetversion}. 
Visual inspection on these and others (not shown)  suggests that our geometric condition holds with high probability.

\begin{figure}[htp]
  \centering
  \begin{tabular}{cccc}
    \includegraphics[width=0.245\linewidth]{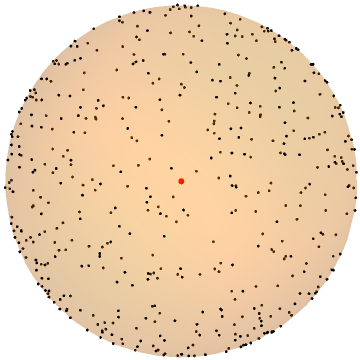}
    \includegraphics[width=0.245\linewidth]{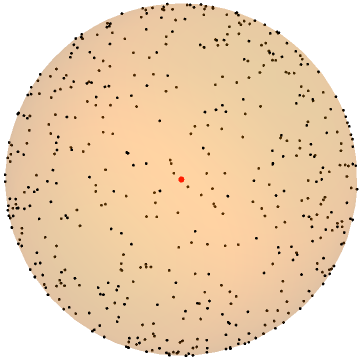}
    \includegraphics[width=0.245\linewidth]{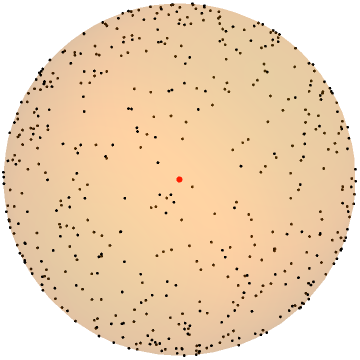}
    \includegraphics[width=0.245\linewidth]{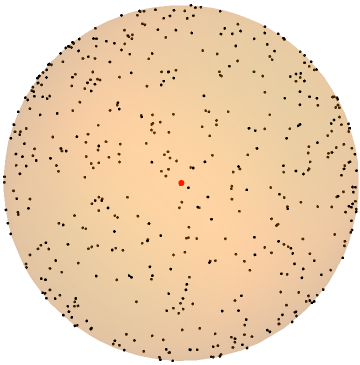}\\
    \includegraphics[width=0.245\linewidth]{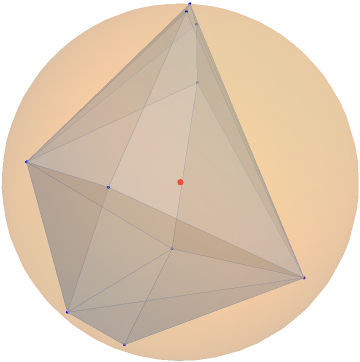}
    \includegraphics[width=0.245\linewidth]{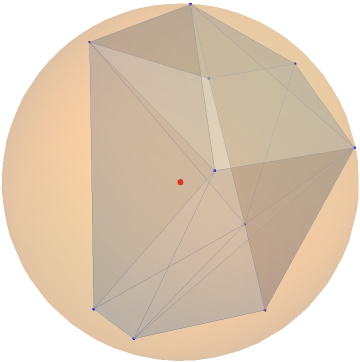}
    \includegraphics[width=0.245\linewidth]{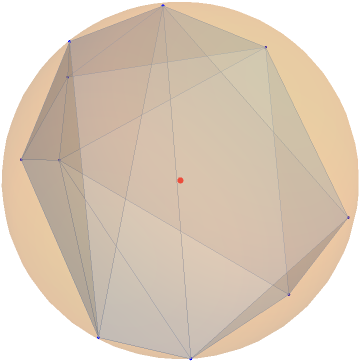}
    \includegraphics[width=0.245\linewidth]{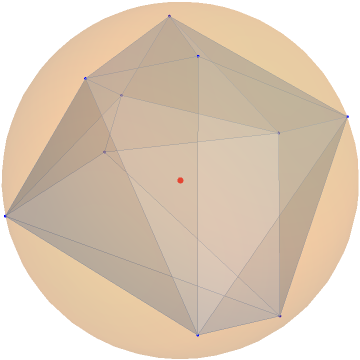}
  \end{tabular}
  \caption{{\bf Top row}: Projections onto $\mathcal{S}^2$ (inside  randomly selected 3D subspaces) of weight vectors in the first fully connected layer of a trained LeNet. {\bf Bottom row}: Projections onto $\mathcal{S}^2$ (inside  randomly selected 3D subspaces) of weight vectors and their convex hull in the second fully connected layer of a trained LeNet.}
  \label{gc_mnist_3d_fc1}
\end{figure}

\section{Conclusions and Future Work}
The slow and fast dynamics of neural network weights under gradient descent is critical for understanding the learning process. 
We performed the first theoretical study on training neural networks to classify linearly inseparable data sets away from the over-parametrized regime. 
We discovered a two time-scale phenomenon of network weights during gradient descent training: a slow phase where the weights spread out to satisfy a geometric condition, and a subsequent fast phase where the weights converge to a global minimum. Both the two scale dynamics and 
geometric conditions are supported by LeNet-5 training on MNIST data. 
\medskip

A future direction is to provide a concrete relation between the number of weights and the rate of convergence, and quantify the effect of over-parameterization on the rate of convergence. 
Another direction 
is to devise efficient method to verify the geometric condition computationally for weight vectors in training deep neural networks on general linearly non-separable data sets.

\section{Acknowledgement}
This work was partially supported by NSF grants IIS-1632935, DMS-1854434, DMS-1924548, and DMS-1924935.

\bibliographystyle{AIMS}
\bibliography{reference}

\medskip
Received xxxx 20xx; revised xxxx 20xx.
\medskip

\end{document}